\documentclass{article}

\usepackage{iclr2027_conference,times}
\iclrfinalcopy 
 \usepackage{wrapfig}

\usepackage[utf8]{inputenc} 
\usepackage[T1]{fontenc}    
\usepackage{iftex}
\ifPDFTeX\else
  \usepackage{newunicodechar}
  \newunicodechar{–}{\textendash}
  \newunicodechar{’}{\textquoteright}
\fi
\usepackage[hidelinks]{hyperref}       
\usepackage{url}            
\usepackage{booktabs}       
\usepackage{amsfonts}       
\usepackage{nicefrac}       
\usepackage{microtype}      
\usepackage{xcolor}         
\usepackage[normalem]{ulem} 
\usepackage{amsmath, amssymb, amsthm}
\newtheorem{theorem}{Theorem}
\newtheorem{lemma}{Lemma}

\newtheorem{definition}{Definition}

\usepackage[ruled,vlined]{algorithm2e}
\SetAlgoCaptionLayout{justified}
\usepackage[noend]{algpseudocode}
\usepackage{enumitem}
\usepackage{graphicx}

\newtheorem{corollary}[theorem]{Corollary}

\let\originalenddefinition\enddefinition
\let\originalendcorollary\endcorollary
\let\originalendtheorem\endtheorem
\renewcommand{\enddefinition}{\originalenddefinition\vspace{-3pt}}
\renewcommand{\endcorollary}{\originalendcorollary\vspace{-3pt}}
\renewcommand{\endtheorem}{\originalendtheorem\vspace{-3pt}}

\newcommand{\Est}{\mathcal E}
\newcommand{\Exp}{\mathbb E}
\DeclareMathOperator{\nnz}{nnz}
\newcommand{\BigO}{\mathcal O}
\newcommand{\R}{\mathbb R}
\newcommand{\tr}{\operatorname{tr}}

\usepackage[hang,flushmargin]{footmisc}

\makeatletter
\renewcommand{\section}{\@startsection{section}{1}{\z@}%
  {\glueexpr -2.0ex plus -0.5ex minus -0.2ex + 2pt\relax}%
  {\glueexpr 1.5ex plus 0.3ex minus 0.2ex - 3pt\relax}%
  {\large\sc\raggedright}}
\renewcommand{\subsection}{\@startsection{subsection}{2}{\z@}%
  {\glueexpr -1.8ex plus -0.5ex minus -0.2ex + 1pt\relax}%
  {\glueexpr 0.8ex plus 0.2ex - 2pt\relax}%
  {\normalsize\sc\raggedright}}
\makeatother

\title{
Dynamic estimation of\\slowly varying sequences
}

\author{Prashant Gokhale, Mikhail Khodak, Sandeep Silwal \\
  University of Wisconsin--Madison \\
  \texttt{\{prashant.gokhale,khodak,silwal\}@wisc.edu}
}

\begin{document}

\maketitle
\fancyhead{}

\begin{abstract}

We consider the problem of sequentially approximating a target function of each element in a slowly-varying sequence, i.e. one where the magnitude $\alpha_i$ of the difference between the elements at positions $i$ and $i-1$, measured in an appropriate norm, is small.
Recent work on implicit trace estimation shows that when $\alpha_i$ is small in Frobenius norm, reusing queries to past sequence elements can reduce the overall cost~\citep{dynamictrace,woodruff2022optimal}.
We introduce a framework that generalizes this approach to linear and nonlinear functions on normed vector spaces, obtaining sequential estimation results for matrix powers, spectral densities, Monte Carlo integration, and certain boundary value problems from partial differential equations~(PDEs).
Furthermore, we develop an algorithm that locally scales the estimation budget with $\alpha_i$, obtaining sharper path-length-style variation bounds of the form $\BigO\left(\tfrac1{\epsilon^2}\sum_{i=1}^m \alpha_i\right)$ on the cost of estimating a sequence of length $m$ to accuracy $\epsilon$.
This improves upon the earlier implicit trace estimation dependence on $\tfrac m{\epsilon^2}\max_i\alpha_i$~\citep{dynamictrace}, which results from fixing the query budget using the worst-case step size $\alpha_i$ and is thus inefficient for stable sequences with rare bursts.
Furthermore, while prior theoretical guarantees assume a known bound on the step sizes, we show in certain cases how to estimate them on the fly with (nearly) no added cost. Lastly, we complement our theoretical results with experiments demonstrating improved accuracy–cost trade-offs over prior baselines on trace estimation and Hessian tracking. In summary, our framework makes sequential approximation general-purpose and adaptive while sharpening state-of-the-art guarantees for dynamic trace estimation.\looseness-1

\end{abstract}

\begingroup
\setlength{\abovedisplayskip}{3pt plus 2pt minus 1pt}
\setlength{\belowdisplayskip}{3pt plus 2pt minus 1pt}
\setlength{\abovedisplayshortskip}{0pt plus 2pt}
\setlength{\belowdisplayshortskip}{2pt plus 1pt minus 1pt}
\section{Introduction}
\label{sec:introduction}

In many areas of theory and practice, we face a slowly evolving system and wish to estimate its properties given only limited, structured access to its state.
Examples of this problem include estimating evolving probability distributions given sample access~\citep{mazzetto2023nonparametric,mazzetto2024improved}, tracking connectivity and centrality metrics in evolving social networks given only local subgraphs~\citep{estrada,graphconn,graph1,graph2,graph3}, maintaining statistics of dynamic data streams~\citep{stat1,stat2,stat3}, or tracking the loss curvature during neural network optimization~\citep{losslandscape,yao2020pyhessian,pmlr-v97-ghorbani19b}.\looseness-1

We propose a general framework for such estimation problems when the underlying sequence consists of $m$ elements of a vector space and the distance between consecutive elements $i-1$ and $i$ is $\alpha_i$. This setting is inspired by past work on estimating traces given only matrix-vector product queries and a bound on the maximum Frobenius norm between consecutive matrices~\citep{dynamictrace,woodruff2022optimal}; 
our framework enables us to find direct applicability to other vector spaces (e.g. $L^\infty$ functions), derive results for certain nonlinear maps (e.g. eigenvalue distributions), develop a novel algorithm whose cost scales with $\sum_i\alpha_i$ (instead of prior works' looser bound of $m\cdot\max_i\alpha_i$), and design an efficient mechanism that estimates $\alpha_i$ on-the-fly (instead of assuming a known global bound).
To achieve this, we require only a linear estimator that has sub-exponential concentration around its target property for every element of the vector space, with concentration parameters that scale with the element's norm.
Such estimators exist in multiple settings and yield results for both linear and nonlinear properties, as evidenced in our contributions:\looseness-1

\begin{enumerate}[leftmargin=*]
\item \textbf{General framework:}
    We identify the existence of what we term a {\em norm-concentrated estimator} of the target quantity, one whose concentration scales with the norm of the input, as a key property that enabled dynamic trace estimation in past work.
    This leads to a meta-algorithm and associated guarantees for sequential stochastic approximation that can be applied to a wide range of both linear and nonlinear problems, including traces of matrix powers, spectral densities, function integrals, and solutions to certain boundary value PDEs.
\item \textbf{Efficiency:}
    Our algorithm dynamically adjusts the budget given to the norm-concentrated estimator based on the local change $\alpha_i$ between consecutive elements, resulting in a total cost that scales with the sum $\sum_{i=1}^m\alpha_i$ of the local changes rather than the worst case change $m\cdot\max_i\alpha_i$.
    This leads to a direct improvement in the query complexity of dynamic trace estimation relative to past work~\citep{dynamictrace}.
\item \textbf{Adaptivity:}
    We identify {\em auto-concentration}, an extension of norm-concentration, as a property that allows us to estimate the local change $\alpha_i$ on-the-fly with (nearly) no added cost.
    This enables us to remove the assumption of a known bound on the step sizes $\alpha_i$ in dynamic trace estimation while remaining provably accurate, another improvement over past work~\citep{dynamictrace}.\looseness-1
\item \textbf{Empirical results:}
    Using experiments on synthetic matrix sequences and Hessians from neural network optimization trajectories, we show that our algorithm requires significantly fewer samples than prior methods for dynamic trace estimation.
    Our approach is simple to implement and can be applied to other estimators as well.\looseness-1
\end{enumerate}

\textbf{Related Work.} 
A representative problem in this setting is trace estimation in the implicit matrix model. Here, we have {\bf matrix-vector product}~(MVP) access to a sequence of matrices $A_1,\dots,A_m$ (for simplicity, assume $\|A_i\|_F \le 1$ after scaling), and we wish to maintain an approximation of $\text{tr}(A_i)$ up to additive $\epsilon$ error every step while minimizing the total number of MVPs. This problem has received a lot of recent attention due to its importance in machine learning and data science, such as when analyzing moments of the Hessian spectral density during neural network optimization~\citep{yao2020pyhessian,pmlr-v97-ghorbani19b}, counting triangles in massive, dynamically changing graphs~\citep{avron2010counting,estrada}, while training Gaussian processes~\citep{gaussian1, gaussian2}, and so on. The assumption that consecutive matrices are close is natural in these settings: a training-loss Hessian evolves gradually, a graph receives a few edge updates per round, a covariance matrix receives rank-one modifications, etc. In such smoothly evolving settings, naively applying a static estimator, such as Hutchinson’s estimator for traces~\citep{hutchinson1989stochastic}, at each step leads to an undesired multiplicative $O(m)$ overhead. Recent work has shown how to dramatically improve on this~\citep{dynamictrace,woodruff2022optimal}, obtaining query complexity that scales with $m$ times the largest local change, $m\cdot\max_i\alpha_i$, where $\alpha_i := \|A_i - A_{i-1}\|_F$.
We obtain query complexity scaling with $\sum_i\alpha_i$, a natural measure of sequence variation similar to path-length-style bounds in online learning~\citep{auer2002exp3,zinkevich2003oco}, and show how their approach can be broadly generalized to other sequential estimation problems.

\section{The dynamic estimation problem}
\label{sec:genframework}

We first set up the general problem and specify our assumptions regarding the static estimator.
Let $V$ be a normed vector space with norm $\|\cdot\|_V$, and consider a linear map $L:V\to\R$ that we {\em cannot} compute directly.
We are interested in estimating $L(v_i)$ for each element in a sequence $v_1,\dots,v_m\in V$ given some setting-specific access model (e.g. MVPs for $V=\R^{n\times n}$) that can be used by some randomized estimator that we will discuss shortly.
Our goal is to adapt to sequences that change relatively slowly between consecutive steps, which we will quantify by the local {\bf step size}
\begin{equation}
    \alpha_i=\|v_i-v_{i-1}\|_V
    \qquad\text{for}\qquad i=1,\dots,m\quad\textrm{and}\quad v_0=0
\end{equation}
We treat these local step sizes as known in Section~\ref{sec:knownstepsize} and generalize to the {\em unknown} case in Section~\ref{sec:unknownstepsize}.
As our query model, we assume access to a randomized estimator $\Est(v,k)$ of $L(v)$ that queries elements $v\in V$ with budget $k>0$ using the setting-specific access model.
Then the {\bf dynamic estimation problem} is to output an estimate $\tilde L_i$ of $L(v_i)$ at each time $i=1,\dots,m$ using only queries of the form $\Est(u_i,k_i)$, where $k_i>0$ is the query budget at time $i$ and $u_i\in V$ is a linear combination of $v_1,\dots,v_i$ with query cost $q>0$;
the goal is to minimize the total query cost $q\sum_{i=1}^mk_i$ while guaranteeing w.p. $\ge1-\delta$ that $|\tilde L_i-L(v_i)|\le\epsilon$ simultaneously $\forall~i\in[m]$.
Note that the actual input $u_i\in V$ to the estimator can be {\em implicit}, i.e. we rely on the estimator $\Est(u_i,k_i)$ being computable given only structured query access to the components $v_1,\dots,v_i$ of $u_i$. 

To understand what we need from $\Est$ to solve this problem, we consider the concrete instantiation studied by \citet{dynamictrace}, in which $V=\R^{n\times n}$ has the Frobenius norm and $L$ is the trace operator.
Their algorithm, {\bf DeltaShift}, outputs a sequence of estimates of the traces $L(v_i)$ that are linear combinations of Hutchinson's estimates of the traces $L(u_i)$ of their {\em damped differences} $u_i=v_i-(1-\gamma)v_{i-1}\in V$ for $\gamma=\alpha$, where $\alpha=\max_i\alpha_i$.
To combine these estimates, they rely on the sub-exponential concentration of Hutchinson's estimator around its mean, a property we will also require of $\Est$. 
However, most crucial for their result is the fact that the estimator's variance scales not only with (the reciprocal of) the number of queries but also with the norm of the input $u_i$, which is why estimating the differences $L(u_i)$ allows them to adapt to small $\alpha$.\footnote{Observe that $\|u_i\|_V\le\|v_i-v_{i-1}\|_V+\gamma\|v_{i-1}\|_V\le\max_i\|v_i-v_{i-1}\|_V+\gamma=2\alpha$.}
We formalize this in the following definition, which underlies the key assumption of our framework.

\begin{definition}[Norm-Concentrated Estimator]
\label{def:norm_concentrated}
A randomized estimator $\Est(v,k)$ is {\bf norm-concentrated with power $p>0$} if there exist constants $c_1,c_2,c_3\ge0$ such that for all elements $v\in V$ and integer budget parameters $k\ge1$ the random variable $\Est(v,k)$ is subexponentially distributed with parameters $\nu$ and $\beta$ satisfying
\begin{equation}
    \nu \le \frac{c_1\|v\|_V^p}{\sqrt{k}}
    \quad \text{and} \quad
    \beta \le \frac{c_2\|v\|_V^p+c_3}{k}
\end{equation}
i.e. $\Exp\left[\exp\left(\lambda(\Est(v,k)-\mu)\right)\right]\le\exp\left(\tfrac12\lambda^2\nu^2\right)~\forall~|\lambda|<1/\beta$, where $\mu=\Exp[\Est(v,k)]$.
Unless otherwise specified, we will use ``norm-concentrated'' to refer to the case $p=1$.
\end{definition}

While Hutchinson's trace estimator is the motivating example of a norm-concentrated estimator, another basic example is the Monte Carlo estimator for integrals. As we show later, quantities as diverse as spectral densities and boundary value solutions can also be estimated by extending these approaches.
The assumption is also reminiscent of the strong growth (or
relative noise) condition in stochastic optimization, under which an
unbiased estimator $g(x,\xi)$ of the gradient $\nabla f(x)$ satisfies
$\Exp_\xi\|g(x,\xi)\|^2 \le C\|\nabla f(x)\|^2$ for constant $C \ge1$
\citep{schmidt2013fast,vaswani2019fast}. Both conditions require the stochastic
variability to shrink with the magnitude of the underlying input.\looseness-1

\begin{table}[!t]
    \centering
    \caption{
    Overview of the dynamic estimation applications we consider and their framework-specific details.
    While nonlinear in the underlying object, both traces of matrix powers and spectral density estimation can be reduced to dynamic trace estimation over $A_i^p$;
    the Dirichlet solution can be represented as linear in $g_i$.
    The norm column gives the norm defining the step sizes $\alpha_i$ for each application.\looseness-1}
    \label{tab:framework_mapping}
    \renewcommand{\arraystretch}{1.5}
    \setlength{\tabcolsep}{4pt}
    \resizebox{\linewidth}{!}{
    \begin{tabular}{@{}l l l l l@{}}
        \toprule
        \textbf{Application} & \textbf{Underlying Object} & \textbf{Map} & \textbf{Norm} & \textbf{Estimator} \\
        \midrule
        Trace Estimation (\S\ref{sec:trace_estimation}) & matrix $A_i\in\R^{n\times n}$ & $\tr(A_i)$ & $\|\cdot\|_F$ & Hutchinson's \\
        -- Matrix Power (\S\ref{sec:trace_powers}) & matrix $A_i\in\R^{n\times n}$ & $\tr(A_i^p)$ & $\|\cdot\|_F$ & Hutchinson's \\
        -- Spectral Density (\S\ref{sec:spectral_density}) & symmetric $A_i\in\R^{n\times n}$ & $\tfrac1{\sqrt n}\tr(A_i^p),p=1,\dots,\BigO(\tfrac1\epsilon)$ & $\|\cdot\|_F$ & Hutchinson's \\
        \addlinespace
        Function Integration (\S\ref{sec:monte_carlo}) & $f_i:\Omega\mapsto\R,\|f_i\|_\infty\le1$ & $\Exp_{X\sim P}[f_i(X)]$ & $\|\cdot\|_2$ & Monte Carlo \\
        -- Dirichlet Problem (\S\ref{sec:dirichlet}) & $g_i:\partial\Omega\mapsto\R,\|g_i\|_\infty\le1$ & $u(x)$ s.t. $\Delta u|_\Omega=0$ \& $u|_{\partial\Omega}=g_i$ & $\|\cdot\|_2$ & Monte Carlo \\
        \bottomrule
    \end{tabular}}
\end{table}

\subsection{Dynamic estimation with known step sizes}
\label{sec:knownstepsize}

The naive approach to solving the dynamic estimation problem is to compute $\Est(v_i,k)$ at every step $i$ with a sufficiently large fixed budget $k$ to guarantee $|\Est(v_i,k)-L(v_i)|\le\epsilon$ w.p. $\ge1-\tfrac\delta m$ and then apply a union bound, which incurs a total budget of $\mathcal{O}\left(\tfrac m{\epsilon^2}\log\tfrac m\delta\right)$.
For traces, DeltaShift improves on this by aggregating estimates $\Est(u_i,k)$ of the damped differences $u_i=v_i-(1-\gamma)v_{i-1}$ described before while using a fixed budget $k=\BigO\left(\tfrac\alpha{\epsilon^2}\log\tfrac m\delta\right)$ queries at every step after initialization, thus attaining a total cost $\BigO\left(\tfrac{1+m\alpha}{\epsilon^2}\log\tfrac m\delta\right)$. This approach is conceptually related to certain recurrent variance-reduction estimators in stochastic optimization~\citep{nguyen2017sarah,fang2018spider}, which similarly maintain a state and estimate only a small residual at each step.\looseness-1

We now improve upon this even further by dynamically rescaling both the damping factor and the query budget at each time step to account for the local step size $\alpha_i$, as described in Algorithm~\ref{alg:adaptive_algo};
 this yields the following path-length-style bound on the total cost:

\begin{theorem}[Dynamic Sample Complexity, known step sizes]
\label{thm:main,knownstep}
Suppose in Algorithm~\ref{alg:adaptive_algo} we have $\|v_i\|_V\le1$ and $\alpha_i\ge\|v_i-v_{i-1}\|_V$ for all $i\in[m]$, and let $\Est$ be a norm-concentrated unbiased estimator of $L$ that Algorithm~\ref{alg:adaptive_algo} invokes with independent randomness at every step $i$. 
Then  w.p. $\ge1-\delta$ its estimates $\tilde L_1,\dots,\tilde L_m$ satisfy $|\tilde{L}_i-L(v_i)|\le\epsilon$ simultaneously for all $i \in [m]$ and the total budget spent querying $\Est$ (with queries consisting of linear combinations of at most two elements of $V$) is bounded by\looseness-1
\begin{equation}\label{eq:knownstepcost}
    \BigO\left(\left(
\left(\frac{c_1^2}{\epsilon^2}+\frac{c_2}{\epsilon}\right)
\sum_{i=1}^m\alpha_i
+\frac{c_3m}{\epsilon}\right)\log\frac{m}{\delta}+\nnz_i(\alpha_i)
\right)
\end{equation}
calls to $u_i=v_i-(1-\gamma_i)v_{i-1}$, 
where $\nnz_i(\alpha_i)\le m$ is the number of nonzero step size bounds $\alpha_i$.\footnote{The low-order ($\epsilon$-independent) term $\nnz_i(\alpha_i)=\BigO(m)$ is present due to the need to ensure integrality of $k_i$;
it should also appear (but was omitted) in the DeltaShift bound~\citep[Theorem~1.1]{dynamictrace}.\looseness-1}\looseness-1
\end{theorem}

Note that the non-dynamic baseline has sample complexity $\tilde{\mathcal O}(\tfrac m{\epsilon^2})$, so
Theorem~\ref{thm:main,knownstep} improves upon it whenever $\sum_{i=1}^m\alpha_i=o(m)$. More importantly, the total cost in Theorem~\ref{thm:main,knownstep} scales with the sum of local changes $\sum_i\alpha_i$, rather than the worst-case change $m\cdot\max_i\alpha_i$. This yields a direct improvement in the query complexity of dynamic trace estimation over prior work~\citep{dynamictrace}. 
Below we prove this result by induction on the sub-exponential parameters of the estimates at each time, before proceeding to handle the case of unknown step sizes via on-the-fly norm estimation.

\begin{algorithm}[!t]
\DontPrintSemicolon
\caption{AdaDeltaShift}
\label{alg:adaptive_algo}
\KwIn{sequence $v_1,\dots,v_m$, step sizes $\alpha_1,\dots,\alpha_m$, target accuracy $\epsilon$, failure probability $\delta>0$, norm-concentrated estimator $\Est$ with constants $c_1,c_2,c_3$}
$N \gets \frac{\epsilon^2}{2 \log(2m/\delta)},B \gets \frac{\epsilon}{2 \log(2m/\delta)}$\tcp*{set concentration thresholds}
$v_0\gets0,\tilde L_0\gets0$\tcp*{initialize running estimate}
\For{$i=1,\dots,m$}{
    $\gamma_i \gets \min(1,\alpha_i)$\tcp*{set damping factor}
    $u_i\gets v_i-(1-\gamma_i)v_{i-1}$\tcp*{define implicit damped residual}
    $k_i\gets\left\lceil \max\left( \frac{4 c_1^2 \gamma_i}{N}, \frac{2 c_2 \gamma_i + c_31_{\gamma_i>0}}{B} \right) \right\rceil$\tcp*{set query budget}
    $\hat L_i\gets\Est(u_i,k_i)$\tcp*{estimate residual map}
    $\tilde L_i\gets(1-\gamma_i)\tilde L_{i-1}+\hat L_i$\tcp*{update running estimate}
}
\KwOut{sequence of running estimates $\tilde{L}_1, \dots, \tilde{L}_m$}
\end{algorithm}

\begin{proof}
Let $e_i=\tilde L_i-L(v_i)$. We prove inductively that $e_i$ is sub-exponential with parameters $\nu_i^2\le N=\tfrac{\epsilon^2}{2 \log(2m/\delta)}$ and $\beta_i\le B=\tfrac{\epsilon}{2 \log(2m/\delta)}$.
By standard sub-exponential tail bounds (Lemma~\ref{lem:tail_bound}), this guarantees $\mathbb{P}[|e_i| \ge \epsilon] \le 2 \exp\left( - \min\left( \frac{\epsilon^2}{2N}, \frac{\epsilon}{2B} \right) \right) \le \frac{\delta}{m}$.
A union bound over $m$ steps then yields the global success probability of at least $1-\delta$.

The base case follows from $e_0=\tilde L_0-L(v_0)=-L(v_0)$ being a constant, which is sub-exponential with parameters $\nu_0^2=\beta_0=0$.
For $i>0$, linearity of $L$ gives $e_i=(1-\gamma_i)e_{i-1}+\eta_i$, where $\eta_i=\hat L_i-L(u_i)$.
If $\alpha_i\le1$, then $\gamma_i=\alpha_i$ and 
\begin{equation}
    \|u_i\|_V = \|v_i - v_{i-1} + \gamma_i v_{i-1}\|_V \le \|v_{i} - v_{i-1}\|_V + \gamma_i \|v_{i-1}\|_V \le \alpha_i + \alpha_i = 2\alpha_i
\end{equation}
Otherwise, if $\alpha_i>1$, then $\gamma_i=1$ so $\|u_i\|_V=\|v_i\|_V\le1\le2\gamma_i$.
Thus $\|u_i\|_V\le2\gamma_i$ in both cases, so $\eta_i$ is sub-exponential with parameters $\nu^2
\le\tfrac{4c_1^2\gamma_i^2}{k_i}\le\gamma_iN$ and $\beta\le\tfrac{2c_2\gamma_i+c_3}{k_i}\le B$.
Since $\eta_i$ is independent of $e_{i-1}$, by composition of sub-exponential random variables (Lemma~\ref{lem:linear_combinations}) we have that $e_i=(1-\gamma_i)e_{i-1}+\eta_i$ is sub-exponential with parameters
\begin{equation}
    \nu_i^2\le(1-\gamma_i)^2N+\gamma_iN\le N
    \qquad\textrm{and}\qquad
    \beta_i\le\max\{(1-\gamma_i)B,B\}=B
\end{equation}
as desired.
Lastly, to count the number of queries to $u_i=v_i-(1-\gamma_i)v_{i-1}$, we bound $k_i$ to get
\begin{equation}\label{eq:querycount}
    \sum_{i=1}^m k_i
    \le\sum_{i=1}^m
    \frac{4c_1^2\gamma_i}{N}
    +\frac{2c_2\gamma_i+c_31_{\gamma_i>0}}{B}
    +1_{\gamma_i>0}
\end{equation}
Since $\gamma_i\le\alpha_i$, substituting the definitions of $N$ and $B$ yields the stated bound~\eqref{eq:knownstepcost}.
\end{proof}

\subsection{Dynamic estimation with unknown step sizes}
\label{sec:unknownstepsize}

In the preceding analysis, we assumed explicit knowledge of the local step sizes $\alpha_i=\|v_i-v_{i-1}\|_V$ to set our parameters.
While in practice $\alpha_i$ is often unknown, as we show next it can often be efficiently estimated on-the-fly.
In fact, if distances between elements can be approximated with the following specific kind of norm-concentrated estimator, then Algorithm~\ref{alg:adaptive_algo} can be run with cheap over-estimates of $\alpha_i$ that suffice to (nearly) preserve the guarantees of Theorem~\ref{thm:main,knownstep}.

\begin{definition}[Auto-Concentrated Estimator]
\label{def:auto_concentrated}
    A nonnegative, unbiased estimator $\Est_V(v,k)$ of the squared norm $\|v\|_V^2$ is {\bf auto-concentrated} if there exist constants $a_1,a_2>0$ such that $\Est_V$ is norm-concentrated with power $p=2$ and constants $a_1,a_2,0$ (corresponding to $c_1,c_2,c_3$ in Definition~\ref{def:norm_concentrated}).
\end{definition}

The natural case of an auto-concentrated estimator is Hutchinson's estimator for the squared Frobenius norm, which is derived from Hutchinson's trace estimator.
Such estimators give probabilistic upper bounds on $\alpha_i$ at each step $i$ by fixing some $\rho\in(0,1)$ and calling $\Est_V(v_i-v_{i-1},r)$ with $r=\lceil2\max\{\tfrac{a_1^2}{\rho^2},\tfrac{a_2}\rho\}\log\tfrac{4m}\delta\rceil$ queries to $v_i-v_{i-1}$.
Sub-exponential concentration (Lemma~\ref{lem:tail_bound}) then implies that for $\alpha_i>0$ we have
\begin{equation}
    \mathbb P\left[|\Est_V(v_i-v_{i-1},r)-\alpha_i^2|>\rho\alpha_i^2\right]
    \le
    2\exp\left(
    -\min\left\{\frac{\rho^2r}{2a_1^2},\frac{\rho r}{2a_2}\right\}
    \right)
    \le\frac{\delta}{2m}
\end{equation}
while for $\alpha_i=0$ we have by unbiasedness and nonnegativity that $\Est_V(v_i-v_{i-1},r)=0$ a.s.
Thus w.p. $\ge1-\tfrac\delta2$ we have $|\Est_V(v_i-v_{i-1},r)-\alpha_i^2|\le\rho\alpha_i^2$ simultaneously for all $i$, which is equivalent to
\begin{equation}\label{eq:overestimate}
    \alpha_i\sqrt{1-\rho}
    \le\sqrt{\Est_V(v_i-v_{i-1},r)}
    \le\alpha_i\sqrt{1+\rho}
\end{equation}
This allows us to run Algorithm~\ref{alg:adaptive_algo} with cheap estimates $\tilde\alpha_i$ that are probabilistically guaranteed to be a small multiplicative factor greater than the true step sizes $\alpha_i$, as described in the following result:

\begin{theorem}[Dynamic Sample Complexity, unknown step sizes]
\label{thm:unknown_step_sizes}
Under the assumptions of Theorem~\ref{thm:main,knownstep}, suppose we additionally have an estimator $\Est_V$ of $\|\cdot\|_V^2$ that is auto-concentrated with constants $a_1,a_2$ and that we invoke it independently of $\Est$ at every step $i=1,\dots,m$ to run Algorithm~\ref{alg:adaptive_algo} but with failure probability $\tfrac\delta2$ and setting step size bounds $\alpha_i$ via the following estimates for fixed $\rho\in(0,1)$:\looseness-1
\begin{equation}
    \tilde\alpha_i=\sqrt{\tfrac{\Est_V(v_i-v_{i-1},r)}{1-\rho}}
    \qquad\textrm{for}\qquad
    r=\left\lceil2\max\left\{\tfrac{a_1^2}{\rho^2},\tfrac{a_2}\rho\right\}\log\tfrac{4m}\delta\right\rceil
\end{equation}
Then w.p. $\ge1-\delta$, the method succeeds (i.e. $|\tilde L_i-L(v_i)|\le\epsilon~\forall~i\in[m]$) and uses a query budget of
\begin{equation}\label{eq:unknownstepcost}
    \BigO\left(\left(
\left(\frac{c_1^2}{\epsilon^2}+\frac{c_2}{\epsilon}\right)
\sqrt{\frac{1+\rho}{1-\rho}}\sum_{i=1}^m\alpha_i
+\frac{c_3m}{\epsilon}\right)\log\frac{m}{\delta}+\nnz_i(\alpha_i)
\right)
\end{equation}
for $\Est$ at $v_i-(1-\gamma_i)v_{i-1}$ and a total query budget of 
$\BigO\left(\left(\tfrac{a_1^2}{\rho^2}+\tfrac{a_2}\rho\right)m\log\tfrac m\delta\right)$ for $\Est_V$ at $v_i-v_{i-1}$.
\end{theorem}

For constant $a_1,a_2,\rho$, Theorem~\ref{thm:unknown_step_sizes} only has an extra overhead of $\BigO(m\log\tfrac m\delta)$ compared to Theorem~\ref{thm:main,knownstep}, with no added dependency on $\epsilon$.
Thus, given an auto-concentrated estimator of the norm of the sequence elements, we obtain adaptivity to unknown step sizes nearly for free.
The parameter $\rho$ controls a tradeoff between the two query budgets: smaller $\rho$ gives tighter upper bounds on the step sizes but requires more queries to $\Est_V$, while larger $\rho$ reduces the norm-estimation budget at the cost of a larger factor $\sqrt{(1+\rho)/(1-\rho)}$ multiplying the adaptive term in the budget for $\Est$. In our applications, we can think of $\rho$ as a fixed constant.


\begin{proof}
    Success follows from the union of (a)~Equation~\ref{eq:overestimate}'s guarantee that $\tilde\alpha_i\ge\alpha_i\ge\|v_i-v_{i-1}\|_V$ holds for all $i\in[m]$ w.p. $\ge1-\tfrac\delta2$ with (b)~Theorem~\ref{thm:main,knownstep}'s guarantee with failure probability $\tfrac\delta2$.
    The number of queries to $\Est$ is given in Equation~\ref{eq:knownstepcost} followed by the bound $\gamma_i\le\tilde\alpha_i\le\alpha_i\sqrt{(1+\rho)/(1-\rho)}$, while the number of queries to $\Est_V$ is simply the number of queries $mr$ to $v_i-v_{i-1}$. 
\end{proof}


\section{Application: Dynamic estimation of matrix properties}\label{sec:trace_estimation}

We first apply our general framework to obtain improved guarantees for dynamic matrix trace estimation in the matrix-vector product model, as introduced in Section~\ref{sec:introduction}.
In this application, the vector space $V=\R^{n\times n}$ is equipped with the Frobenius norm and the linear map is $L(A)=\tr(A)$.
Afterwards, we work in the same MVP model to show results for dynamic estimation of two related nonlinear quantities: 
traces of matrix powers and spectral densities.

To query matrix $A\in\R^{n\times n}$, we will use the standard randomized algorithm for implicit trace estimation, Hutchinson's estimator~\citep{hutchinson1989stochastic}.
For a budget of $k$ queries, it is defined as
\begin{equation}
    H_k(A) := \frac{1}{k}\sum_{j=1}^k g_j^\top A g_j
    \qquad
    \textrm{for i.i.d. Rademacher vectors}
    \qquad
    g_1,\dots,g_k\in\{\pm1\}^n
\end{equation}
It is well-known that Hutchinson's estimator is unbiased and sub-exponential~\citep{rudelson2013hansonwright} with parameters $\nu\le c_1\|A\|_F/\sqrt k$ and $\beta\le c_2\|A\|_F/k$ for some constants $c_1,c_2\ge0$, making it a norm-concentrated estimator by Definition~\ref{def:norm_concentrated}.\footnote{Specifically, use Claim~B.3 of \citet{dynamictrace}, then bound the spectral norm by the Frobenius.\looseness-1}

In addition to its use for traces, Hutchinson's is also well-known to yield an unbiased estimator of the squared Frobenius norm by applying it to the positive semi-definite~(PSD) Gram matrix $A^\top A$ and using $\|A\|_F^2=\tr(A^\top A)$.
The Gram matrix being PSD means its Frobenius norm is bounded by its trace and thus $\|A^\top A\|_F\le\tr(A^\top A)=\|A\|_F^2$, so $H_k(A^\top A)$ is {\em also} sub-exponential but with parameters $\nu\le c_1\|A\|_F^2/\sqrt k$ and $\beta\le c_2\|A\|_F^2/k$.
With the nonnegativity of $H_k(A^\top A)$ (which also follows from PSD), $H_k(A^\top A)$ is thus an auto-concentrated estimator of $\|A\|_F^2$ by Definition~\ref{def:auto_concentrated}.

For the trace-estimation and matrix-power bounds below, we fix $\rho\in(0,1)$, and the hidden constants may depend on $\rho$.
We can thus run Algorithm~\ref{alg:adaptive_algo} in the manner described in Theorem~\ref{thm:unknown_step_sizes}, with $H_k$ applied to both the damped differences $U_i=A_i-(1-\gamma_i)A_{i-1}$ and to the Gram matrices $\Delta_i^\top\Delta_i$ of the differences $\Delta_i=A_i-A_{i-1}$.
Since both $H_k(U_i)$ and $H_k(\Delta_i^\top\Delta_i)$ cost $2k$ MVPs each, their respective norm-concentration and auto-concentration imply the following corollary:

\begin{corollary}[Dynamic Trace Estimation, unknown step sizes]\label{cor:trace}
Suppose we are given a sequence $A_1, \dots, A_m\in\R^{n\times n}$ with $A_0=0$, $\|A_i\|_F\le1$ for all $i\in[m]$, and unknown $\alpha_i=\|A_i-A_{i-1}\|_F$ and we run Algorithm~\ref{alg:adaptive_algo} as described in Theorem~\ref{thm:unknown_step_sizes}, with Hutchinson's trace and squared-norm estimators as $\Est$ and $\Est_V$, respectively.
Then w.p. $\ge1-\delta$ the resulting trace estimates $\tilde T_1=\tilde L_1,\dots,\tilde T_m=\tilde L_m$ satisfy $|\tilde T_i-\tr(A_i)|\le\epsilon$ simultaneously for all $i\in[m]$ and the total number of MVPs is bounded by $\BigO\left(\left(m+\tfrac1{\epsilon^2}\sum_{i=1}^m\alpha_i\right)\log\tfrac m\delta\right)$.
\end{corollary}

The bound of \citet{dynamictrace} for known steps is
$\BigO\left(\tfrac1{\epsilon^2}\log\tfrac m\delta\left(1+m\cdot\max_i\alpha_i\right)\right)$.
As noted earlier, our result replaces the $m\cdot\max_i\alpha_i$ term with the sharper quantity $\sum_{i=1}^m \alpha_i$. Moreover, it provides a way to estimate the step sizes on the fly with only an $\epsilon$-independent $m\log\tfrac m\delta$ overhead.

\subsection{Traces of matrix powers}
\label{sec:trace_powers}

We now show that our framework naturally extends to tracking the trace of matrix powers, which we do in the same MVP setting over $V=\R^{n\times n}$ equipped with the Frobenius norm. Fix an integer $p\ge1$. 
While at first glance $\tr(A^p)$ is a nonlinear function of $A$, we can instead apply the previous result (Corollary~\ref{cor:trace}) to the sequence $A_1^p,\dots,A_m^p$ to get an $\epsilon$-accuracy guarantee and a query cost that scales with $\sum_{i=1}^m\|A_i^p-A_{i-1}^p\|_F$ (rather than the desired $\sum_{i=1}^m\alpha_i=\sum_{i=1}^m\|A_i-A_{i-1}\|_F$).
The remaining questions are thus (1)~whether we can get a result in terms of the Frobenius differences $\alpha_i=\|A_i-A_{i-1}\|_F$ instead and (2)~the MVP cost of applying the Hutchinson estimators to $A_i^p-(1-\gamma_i)A_{i-1}^p$ and the Gram matrix of $A_i^p-A_{i-1}^p$.
For the latter, it is easy to see that $H_k$ can be applied to both with $2pk$ MVPs each by simply applying the matrices sequentially, i.e. $A_i^pg=A_i(A_i(\cdots(A_ig)))$. This also avoids explicitly forming potentially dense matrix powers by requiring only matrix-vector access to $A_i$.
For the former, we can use the fact that the change in the $p$th power can be written as $A_i^p-A_{i-1}^p=\sum_{j=0}^{p-1}A_i^j(A_i-A_{i-1})A_{i-1}^{p-1-j}$, so by the sub-multiplicativity of the Frobenius norm ($\|XY\|_F \le \|X\|_2 \|Y\|_F$) and using $\|A_i\|_2\le\|A_i\|_F\le1$, we obtain
\begin{align}\label{eq:powerbound}
    \begin{split}
    \|A_i^p-A_{i-1}^p\|_F
    &\le\min\left\{\|A_i^p\|_F+\|A_{i-1}^p\|_F,\sum_{j=0}^{p-1}\|A_i^j\|_2\|A_i-A_{i-1}\|_F\|A_{i-1}^{p-1-j}\|_2\right\}\\
   &\le \min\left\{\|A_i^{p-1}\|_2\|A_i\|_F+\|A_{i-1}^{p-1}\|_2\|A_{i-1}\|_F,p\alpha_i\right\}
   \le\min\{2,p\alpha_i\}
    \end{split}
\end{align}
We thus obtain the following result directly from Corollary~\ref{cor:trace}, with $\sum_{i=1}^m\alpha_i$ replaced by $\min\{m,p\sum_{i=1}^m\alpha_i\}$ by (1), and an additional factor of $p$ in the MVP cost by (2):

\begin{corollary}[Dynamic Estimation of Traces of Matrix Powers, unknown step sizes]\label{cor:tracepower}
Suppose we are given a sequence $A_1, \dots, A_m\in\R^{n\times n}$ with $A_0=0$, $\|A_i\|_F\le1$ for all $i\in[m]$, and unknown $\alpha_i=\|A_i-A_{i-1}\|_F$ and we run Algorithm~\ref{alg:adaptive_algo} as described in Theorem~\ref{thm:unknown_step_sizes} on the matrix powers $A_1^p,\dots,A_m^p$, with Hutchinson's trace and squared-norm estimators as $\Est$ and $\Est_V$, respectively.
Then w.p. $\ge1-\delta$ the resulting trace power estimates $\tilde T_1=\tilde L_1,\dots,\tilde T_m=\tilde L_m$ satisfy $|\tilde T_i-\tr(A_i^p)|\le\epsilon$ simultaneously for all $i\in[m]$ and the total number of MVPs is bounded by $\BigO\left(\left(m+\tfrac1{\epsilon^2}\min\left\{m,p\sum_{i=1}^m\alpha_i\right\}\right)p\log\tfrac m\delta\right)$.
\end{corollary}

Thus for small path-lengths $\sum_{i=1}^m\alpha_i=\BigO(m\epsilon^2/p)$, the cost (up to log factors) is at most $\BigO(pm)$, while recovering the naive bound of $\BigO(mp/\epsilon^2)$ otherwise.

\subsection{Spectral density estimation}
\label{sec:spectral_density}

We next use estimates of traces of matrix powers to track the spectral distribution of a changing symmetric matrix $A$ with spectral radius at most one, again in the MVP setting. 
This application, called {\em spectral density estimation}~(SDE), arises in dynamic settings such as tracking Hessian spectra during neural network training~\citep{pmlr-v97-ghorbani19b}.
The spectral distribution itself, which corresponds to the uniform distribution over the eigenvalues of $A$ weighted by multiplicity, is nonlinear in $A$;
however, we can use a well-known reduction of SDE to computing normalized traces $\tr(A^p)/n$ of matrix powers, which correspond to the $p$th moments.
In particular, \citet{kong2017spectrum} show that if we have $\epsilon'$-accurate estimates $\hat m_1,\dots,\hat m_K$  of the first $K=\Omega(1/\epsilon)$ moments, i.e. $|\hat m_p-\tr(A^p)/n|\le\epsilon'~\forall~p\in[K]$, then if $\epsilon'=\exp(-\Omega(1/\epsilon))$ we can use them to construct a linear program~\eqref{eq:moment_reconstruction_lp} with $\BigO(1/\epsilon'+K)$ variables and $\BigO(K)$ constraints returning a distribution that is $\epsilon$-close in Wasserstein-1 distance to the true spectral density.
See Theorem~\ref{thm:spectral_moments} in Appendix~\ref{sec:moment_reconstruction} for details.\looseness-1

With the problem reduced to estimating traces of multiple matrix powers, for the sequential case it is natural to take an approach similar to the one in Section~\ref{sec:trace_powers}, adjusted for setting-specific assumptions.
In particular, while we again consider a sequence of matrices $A_1,\dots,A_m$ with $\|A_i-A_{i-1}\|_F\le\alpha_i$, we now assume that they are symmetric and have spectral radius at most one, i.e. that their eigenvalues are in $[-1,1]$.
This implies that $\|A_i^p\|_F^2=\sum_{j=1}^n|\lambda_j(A_i)|^{2p}\le n$ for every $p\ge1$, so the rescaled powers $B_{i,p}=A_i^p/\sqrt n$ satisfy the  requirement $\|B_{i,p}\|_F\le1$ of Corollary~\ref{cor:trace}.
Furthermore, it suffices to release $\epsilon'\sqrt n$-accurate estimates $\tilde L_{i,p}$ of $\tr(B_{i,p})$ to guarantee that the approximate moments $\hat m_{i,p}=\tilde L_{i,p}/\sqrt n$ satisfy the required accuracy:\looseness-1
\begin{equation}
    \left\lvert\hat m_{i,p}-\frac{\tr(A_i^p)}n\right\rvert
    =\frac{|\tilde L_{i,p}-\tr(B_{i,p})|}{\sqrt n}
    \le\epsilon'
\end{equation}
Lastly, analogously to Equation~\ref{eq:powerbound}, we can get $\|B_{i,p}-B_{i-1,p}\|_F\le\min\{2,p\alpha_i/\sqrt n\}$, so using the argument in Corollary~\ref{cor:tracepower} and then taking the maximum over $p\in[K]$ yields the following:
\begin{corollary}[Dynamic SDE, unknown step sizes]
\label{cor:SDE}
Consider a sequence of symmetric matrices $A_1,\dots,A_m\in\R^{n\times n}$ with spectral radii at most one, $A_0=0$, and unknown $\alpha_i=\|A_i-A_{i-1}\|_F$.
Suppose for some $K=\Theta(1/\epsilon)$, $\epsilon'=\exp(-\Theta(C/\epsilon))$, and $C=\BigO(1)$ we run Algorithm~\ref{alg:adaptive_algo} as described in Theorem~\ref{thm:unknown_step_sizes} on each of the $K$ sequences $A_1^p/\sqrt n,\dots,A_m^p/\sqrt n$ for $p=1,\dots,K$, using target accuracy $\epsilon'\sqrt n$ and failure probability $\delta/K$.
Then if at each $i$ we reuse randomness and MVPs across all instantiations $p=1,\dots,K$ of Algorithm~\ref{alg:adaptive_algo}, the total number of MVPs is bounded by
\begin{equation}
    \BigO\left(\left(m+\frac{\exp(C/\epsilon)}n\min\left\{m,\frac1{\epsilon\sqrt n}\sum_{i=1}^m\alpha_i\right\}\right)\frac1\epsilon\log\frac m{\epsilon\delta}\right)
\end{equation}
and w.p. $\ge1-\delta$ we obtain moment estimates $\hat m_{i,p}=\tilde L_{i,p}/\sqrt n$ such that for each $i$ the linear program~\eqref{eq:moment_reconstruction_lp} induced by the moment estimates $\hat m_{i,1},\dots,\hat m_{i,K}$ yields a distribution that is $\epsilon$-close in Wasserstein-1 distance to the spectral distribution of $A_i$.\looseness-1
\end{corollary}

Observe that exponential dependence on $1/\epsilon$ is multiplied by the path-length $\sum_i\alpha_i$, so if the matrices vary slowly the leading-order term is much smaller than the $\tilde\BigO(\tfrac mn\exp(C/\epsilon))$ MVPs used by a non-dynamic algorithm in the same setting.
We note that more efficient methods are available for
\emph{non-dynamic} SDE: Chebyshev moment
matching~\citep{braverman2022sublinear,pmlr-v291-musco25a} and stochastic
Lanczos quadrature~\citep{pmlr-v139-chen21s} achieve Wasserstein-1 error $\epsilon$
using a number of MVPs polynomial in $1/\epsilon$. Adapting these methods
to exploit small changes between successive matrices in the dynamic
setting is an interesting direction for future research.\looseness-1

\section{Application: Dynamic estimation of linear functionals}
\label{sec:function_space}

We next consider a setting where dynamic estimation has yet to be studied:
approximating maps from function spaces to the real line, i.e. {\em functionals}.
The specific applications are (1)~dynamic function integration and (2)~tracking solutions to Laplace's equation with changing boundary conditions. 
Both can be reduced to estimating expectations under a fixed distribution, so if we assume the ability to evaluate the functions in the sequence at any point then we can apply the well-known Monte Carlo estimator.
This section demonstrates the utility of our framework by showing how, once this estimator is shown to be norm-concentrated, it is easy to then also show guarantees for dynamic estimation.\looseness-1

We assume a fixed probability distribution $P$ on a domain $\Omega$ and work with the vector space $V=L^2(\Omega,P)$ of measurable functions on $\Omega$ equipped with the norm $\|f\|_2=\sqrt{\Exp_{X\sim P}[f(X)^2]}$.
Our goal will be to estimate the linear functional $L(f)=\Exp_{X\sim P}[f(X)]$ applied to a sequence of functions $f$ in this space;
we further assume these functions have supremum norm at most one, which also bounds their $\|\cdot\|_2$ norms by one.
To estimate $L(f)$, we will use the standard Monte Carlo estimator\looseness-1
\begin{equation}
    \mathrm{MC}_k(f)
    :=\frac1k\sum_{j=1}^k f(X_j),
    \qquad X_1,\dots,X_k\stackrel{\mathrm{i.i.d.}}{\sim}P.
    \label{eq:mc_estimator}
\end{equation}

It is easy to see that this estimator is unbiased and that its variance is at most $\|f\|_2^2/k$.
For bounded functions as in our applications, Bernstein's
inequality~\citep{bernstein1924modification,vershynin2018high} further implies that it is sub-exponential with parameters $\nu\le\frac{c_1\|f\|_2}{\sqrt k},\beta\le\frac{c_3}{k}$ for absolute constants $c_1,c_3>0$.
Thus, $\mathrm{MC}_k(f)$ is a norm-concentrated estimator with power 1 (and $c_2=0$).



\subsection{Function integration}
\label{sec:monte_carlo}

We first consider a sequence  of integrands $f_1,\dots,f_m$, where the goal is to release estimates of their integrals over a fixed domain.  We can run Algorithm~\ref{alg:adaptive_algo} using $\mathrm{MC}_k$ to estimate the
damped differences $f_i-(1-\gamma_i)f_{i-1}$.
Each sample evaluates the two consecutive functions at the same point, so a
call with budget $k$ costs at most $2k$ function evaluations.
The preceding concentration bounds and Theorem~\ref{thm:main,knownstep}
therefore give the following bound on the dynamic sample complexity. Compared with estimating each integral independently, it replaces the
leading $\frac{m}{\epsilon^2}$ by
$\tfrac1{\epsilon^2}\sum_{i=1}^m\alpha_i$.
Notably, when $\sum_{i=1}^m\alpha_i=\BigO(m\epsilon)$, the cost is
$\BigO((m/\epsilon)\log(m/\delta))$, while the baseline is
$\BigO((m/\epsilon^2)\log(2m/\delta))$.

\begin{corollary}[Dynamic Monte Carlo Integration, known step sizes]
\label{cor:montecarlo}
\label{thm:montecarlo}
Suppose we have sample access to a fixed distribution $P$ and evaluation
access to a sequence $f_1,\dots,f_m\in V$ with $f_0=0$,
$\|f_i\|_\infty\le1$, and known $\alpha_i=\|f_i-f_{i-1}\|_2$.
For $\epsilon,\delta\in(0,1)$, run Algorithm~\ref{alg:adaptive_algo} with
$\Est(f,k)=\mathrm{MC}_k(f)$.
Then w.p. $\ge1-\delta$ the resulting estimates $\tilde L_i$ of the integrals of $f_i$ satisfy
$|\tilde L_i-\Exp_{X\sim P}[f_i(X)]|\le\epsilon$ simultaneously for all
$i\in[m]$, and the total number of samples and function evaluations is
bounded by $
    \BigO\left(
        \left(\frac1{\epsilon^2}\sum_{i=1}^m\alpha_i
        +\frac m\epsilon\right)
        \log\frac{m}{\delta}
    \right)$.
\end{corollary}

\subsection{Dirichlet evaluation}
\label{sec:dirichlet}

We next use dynamic Monte Carlo integration to track solutions to a classical boundary-value PDE known as the Dirichlet problem. 
Given a continuous boundary function $g:\partial\Omega\to\R$ on the boundary of a fixed bounded domain $\Omega\subset\R^d$ satisfying a standard boundary regularity condition, namely the exterior cone condition at every boundary point,
the goal is to find its unique harmonic extension $u\in C(\overline{\Omega})\cap C^2(\Omega)$, i.e. the solution to
\begin{equation}
    \Delta u=0~~\text{in}~~\Omega
    \qquad\text{and}\qquad
    u=g~~\text{on}~~\partial\Omega,
\end{equation}
where $\Delta=\sum_i\partial_{ii}^2$ is the Laplacian operator.
For a fixed interior point $x\in\Omega$, Kakutani's
representation~\citep{kakutani,morters2010brownian} expresses the solution as
\begin{equation}
    u(x)=\int_{\partial\Omega}g(y)P_x(dy)
        =\Exp_{Y\sim P_x}[g(Y)],
    \label{eq:dirichlet_representation}
\end{equation}
where the \emph{harmonic measure} $P_x$ over $\partial\Omega$ is the distribution of the first
boundary point reached by Brownian motion started at $x$.
Crucially, this distribution depends only on the domain and the evaluation
point, not on the boundary function.
Thus, as the boundary condition evolves through $g_1,\dots,g_m$, tracking
the corresponding solutions $u_1(x),\dots,u_m(x)$ is exactly the preceding
integration problem with $P=P_x$ and $f_i=g_i$.
Applying Corollary~\ref{cor:montecarlo} gives:

\begin{corollary}[Dynamic Dirichlet Evaluation, known step sizes]
\label{cor:dirichlet}
Fix $\Omega$ and $x$ as above, and let $g_1,\dots,g_m$ be continuous boundary
functions with $g_0=0$, $\|g_i\|_\infty\le1$, and known
$\alpha_i=\|g_i-g_{i-1}\|_2$.
Suppose we have sample access to $P_x$ and evaluation access to these
functions. For $\epsilon,\delta\in(0,1)$, run
Algorithm~\ref{alg:adaptive_algo} as in Corollary~\ref{cor:montecarlo},
with $P=P_x$ and $f_i=g_i$.
Then w.p. $\ge1-\delta$ the resulting estimates
$\tilde u_i(x)=\tilde L_i$ satisfy
$|\tilde u_i(x)-u_i(x)|\le\epsilon$ simultaneously for all $i\in[m]$,
and the total number of boundary samples and boundary-function evaluations
is bounded by $
    \BigO\left(
        \left(\frac1{\epsilon^2}\sum_{i=1}^m\alpha_i
        +\frac m\epsilon\right)
        \log\frac{m}{\delta}
    \right)$.
\end{corollary}

While mechanically very similar to the previous result, the Dirichlet evaluation result illustrates the wide-ranging applicability of the dynamic estimation problem.
Note also that, in practice, we can use the walk-on-spheres~(WoS) algorithm~\citep{wos} to approximately
sample boundary points from $P_x$ and then evaluate $g$ at those points.
    While they can be expensive on complex geometries, WoS boundary outputs can often be reused across boundary functions, so the reduction in evaluations due to our method would be mainly useful when the boundary function itself is expensive to compute, e.g. when each evaluation requires neural network inference or solving a different PDE.\looseness-1

\vspace{-3pt}

\section{Experiments}
\label{sec:experiments_main}

To validate our improved bounds, we evaluate our adaptive algorithm against the parameter-free DeltaShift baseline of \citet{dynamictrace}.
Figure~\ref{fig:hessian_trace} exhibits our improved path-length bounds, dramatically reducing the number of MVPs needed to estimate the traces of neural network Hessians during training with restarts~\citep{loshchilov2017sgdr}.
See Appendix~\ref{sec:experiments} for full details.\looseness-1

\begin{figure}[!h]
    \centering
    \setlength{\abovecaptionskip}{6pt}
    \includegraphics[width=0.48\linewidth,trim=0 0 535bp 0,clip]{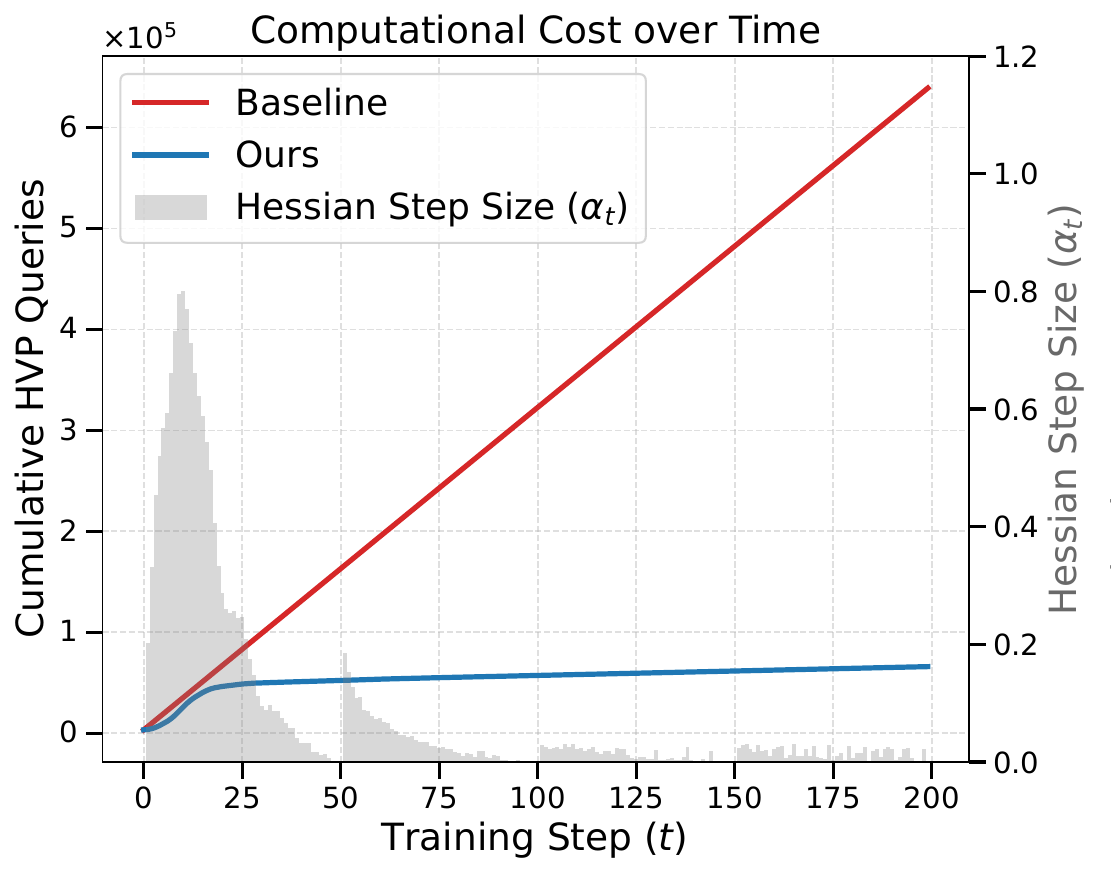}%
    \hfill
    \includegraphics[width=0.48\linewidth,trim=535bp 0 0 0,clip]{hessian_trace_experiment.pdf}
    \caption{\textbf{Hessian Trace Estimation during SGDR.} (Left) Cumulative queries over time, with the shaded background showing the Hessian step size ($\alpha_i$), which spikes heavily every 50 steps due to SGDR restarts. (Right) Empirical trace error, showing that both algorithms maintain roughly the same error, even though our adaptive approach requires far fewer Hessian-vector product queries.}
    \label{fig:hessian_trace}
\end{figure}

\section{Conclusion}
Inspired by prior work, we introduced a general framework for dynamic estimation of slowly-varying sequences in a normed space, identifying \emph{norm-concentration} as a key property that yields sequential estimation results across a wide range of old and new estimation problems for both linear and nonlinear quantities, including traces of matrix powers, spectral densities, function integrals, and solutions to certain boundary value PDEs. Our algorithm dynamically rescales the query budget to account for the local change $\alpha_i$, obtaining sharper path-length-style bounds that replace the $m\cdot\max_i\alpha_i$ dependence with $\sum_{i=1}^m\alpha_i$ in the aforementioned applications.  Given an auto-concentrated estimator of the squared norm, our framework can also estimate the local changes on-the-fly. Experiments on synthetic matrix sequences and Hessians from neural network optimization trajectories demonstrate improved accuracy–cost trade-offs over prior baselines for sequences that are slowly varying. Promising directions for future work include further empirical investigations, e.g. tracking evaluation metrics during pretraining via dynamic function integration, and theoretical extensions to high-dimensional or biased estimators:
the latter could enable dynamic estimation results for other problems with existing estimators exhibiting norm-concentration-like properties, such as quantum shadow tomography~\citep{huang2020predicting} and differentially private covariance estimation~\citep{amin2023private}.\looseness-1

\section{AI Disclosure Statement}

We used large language models to polish the writing, simplify some proofs, assist with conducting experiments, generate experimental code, and support literature review and search.

\endgroup
\bibliographystyle{plainnat}
\bibliography{ref}
\appendix

\newpage
\section{Preliminaries on sub-exponential random variables}
\label{sec:subexp_preliminaries}

\begin{definition}[Sub-Exponential Random Variable]
\label{def:sub_exponential}
A scalar random variable $X$ with mean $\mu$ is sub-exponential with variance parameter $\nu^2 \ge 0$ and scale parameter $\beta \ge 0$ if
\begin{equation}
    \mathbb{E}\!\left[e^{\lambda (X-\mu)}\right] \le e^{\lambda^2\nu^2/2}
    \quad \text{for all } |\lambda| < 1/\beta.
\end{equation}
For $\beta=0$, the condition is required for every $\lambda\in\R$.
Throughout, $1/0$ is interpreted as $+\infty$.
\end{definition}

We use the following two standard facts about sub-exponential random variables.

\begin{lemma}[Sub-Exponential Concentration]
\label{lem:tail_bound}
Let $X$ be a sub-exponential random variable with mean $\mu$, variance parameter $\nu^2$, and scale parameter $\beta$. Then, for any $\epsilon>0$,
\begin{equation}
    \mathbb{P}[|X - \mu| \ge \epsilon]
    \le 2 \exp\left( - \min\left\{
    \frac{\epsilon^2}{2\nu^2}, \frac{\epsilon}{2\beta}
    \right\} \right).
\end{equation}
Consequently, it is sufficient for failure probability at most $\delta$ that
$\nu^2\le\epsilon^2/(2\log(2/\delta))$ and
$\beta\le\epsilon/(2\log(2/\delta))$.
\end{lemma}

\begin{lemma}[Linear Combinations of Sub-Exponentials]
\label{lem:linear_combinations}
Let $X_1,\dots,X_k$ be independent sub-exponential random variables with respective parameters $(\nu_i^2,\beta_i)$.
For $a_1,\dots,a_k\in\R$, the sum $Y=\sum_{i=1}^k a_iX_i$ is sub-exponential with parameters
\begin{equation}
    \nu_*^2 = \sum_{i=1}^k a_i^2 \nu_i^2
    \qquad\text{and}\qquad
    \beta_* = \max_i |a_i| \beta_i.
\end{equation}
\end{lemma}




\section{Spectral reconstruction from approximate moments}
\label{sec:moment_reconstruction}

We now explain how to obtain an approximate spectral distribution of
a symmetric matrix $A\in\R^{n\times n}$ from estimates of normalized
traces $\tr(A^p)/n$ of its powers $p\ge1$.
This relies on the observation that these normalized traces are the
$p$th moments of the spectral distribution, so we can approximate it
by seeking a probability distribution whose moments are close to
these estimates.
Note that the results in this section are not novel, but rather they summarize existing literature~\citep{kong2017spectrum,cohen2017approximating} in our notation.
Throughout this section, we use $W_1(\mu,\nu)$ to denote the $1$-Wasserstein distance, which is the minimum expected distance $\mathbb{E}[|X-Y|]$ over all joint distributions of $(X,Y)$ with marginals $\mu$ and $\nu$.
This distance is widely used to measure error in SDE~\citep{braverman2022sublinear,pmlr-v139-chen21s,pmlr-v291-musco25a}.

To see why this suffices, consider two distributions on $[-1,1]$
that have the same first $K$ moments. For any degree-$K$ polynomial
$P(x)=a_0+\sum_{p=1}^K a_p x^p$, its expectation under either
distribution is determined by these moments.
Thus, if two distributions agree in their first $K$ moments, they
also agree on the expectation of every polynomial of degree at
most $K$.

To relate this observation to the Wasserstein distance $W_1$,
recall its equivalent definition as the supremum, over all
$1$-Lipschitz functions $f$, of the absolute difference between the
expectations of $f$ under the two distributions.
It is also known from Jackson's approximation theorem that for
every such function $f$ on $[-1,1]$ there is a polynomial $P_f$ of
degree at most $K$ satisfying
\begin{equation}
    \sup_{x\in[-1,1]} |f(x)-P_f(x)| \le \frac{C}{K},
\end{equation}
where $C$ is an absolute constant; see
\citet[Section~4.1]{kong2017spectrum}.
Replacing $f$ by $P_f$ therefore changes its expectation under
either distribution by at most $C/K$.
Since the expectations of $P_f$ agree, the expectations of $f$
differ by at most $2C/K$.
This holds for every $1$-Lipschitz function, so the two distributions
are $\mathcal O(1/K)$ apart in $W_1$.

In our setting, the moments are only approximately known.
To handle this,
\citet[Proposition~1]{kong2017spectrum} extend the above argument
to bound the $W_1$ distance in terms of both $K$ and the
discrepancies between the moments.
It therefore suffices to find a distribution whose moments fit
the estimates closely enough.
The reconstruction procedure of
\citet[Algorithm~2, Eq.~(2)]{kong2017spectrum}, also used by
\citet[Section~3.3]{cohen2017approximating}, does this by choosing
probability weights on a fine grid to minimize the total absolute
difference between the resulting moments and their estimates.

Given moment estimates $\widehat m_1,\dots,\widehat m_K$ and a
grid parameter $\epsilon'\in(0,1)$, set
\begin{equation}
    N=\left\lceil\frac{2}{\epsilon'}\right\rceil,
    \qquad
    x_j=-1+\frac{2j}{N},
    \qquad j=0,\dots,N.
\end{equation}
Introducing variables $t_p$ for the absolute moment discrepancies,
we solve the linear program
\begin{equation}
\label{eq:moment_reconstruction_lp}
\begin{aligned}
    \min_{w\ge0,\,t\ge0}\quad
        & \sum_{p=1}^K t_p \\
    \text{subject to}\quad
        & \sum_{j=0}^N w_j=1, \\
         - t_p
         \le & \sum_{j=0}^N  w_jx_j^p-\widehat m_p
        \le t_p,
        \qquad p\in[K].
\end{aligned}
\end{equation}
For an optimal solution $(w^*,t^*)$, let $\widehat\mu$ be the
distribution assigning probability $w_j^*$ to $x_j$.
Combining this reconstruction with
\citet[Proposition~1]{kong2017spectrum} gives the following
consequence, stated here in our notation for completeness;
see also \citep[Lemma~6]{cohen2017approximating}.

\begin{theorem}[SDE from approximate traces; {\citealp{kong2017spectrum}}]
\label{thm:spectral_moments}
There exist universal constants $C_1,C_2>0$ such that the following
holds. Let $A\in\R^{n\times n}$ be a symmetric matrix with all
eigenvalues in $[-1,1]$, and let $\mu$ be its spectral distribution.
For any $\epsilon\in(0,1)$, set
\begin{equation}
    K=\left\lceil\frac{C_1}{\epsilon}\right\rceil,
    \qquad
    \epsilon'=\exp(-C_2/\epsilon).
\end{equation}
Given moment estimates $\widehat m_1,\dots,\widehat m_K$ satisfying
\begin{equation}
    \left|\widehat m_p-\frac{\tr(A^p)}{n}\right|
    \le\epsilon',
    \qquad p\in[K],
\end{equation}
solving the linear program~\eqref{eq:moment_reconstruction_lp}
with grid parameter $\epsilon'$ gives a probability distribution
$\widehat\mu$ on $[-1,1]$ satisfying
$W_1(\widehat\mu,\mu)\le\epsilon$.
\end{theorem}

Applying this reconstruction at each time step gives the spectral
estimates used in Section~\ref{sec:spectral_density}, without any
additional MVPs.

\section{Experiments} \label{sec:experiments}

\textbf{Cost across a sequence and ablations.}
In our first experiment, we generate a synthetic sequence of $2{,}000\times2{,}000$ matrices over $m=200$ time steps. We introduce rare, large perturbations to simulate a stable sequence punctuated by bursts. As shown in Figure~\ref{fig:cost_vs_time}, the baseline method chooses its fixed query budget using the worst-case step size $\max_i\alpha_i$. Our adaptive algorithm instead scales down its query budget during stable periods, resulting in a lower total MVP cost. Figure~\ref{fig:ablations} shows that this reduction persists across a wide range of error tolerances $\epsilon$ and failure probabilities $\delta$.

\begin{figure}
    \centering
    \includegraphics[width=0.90\linewidth]{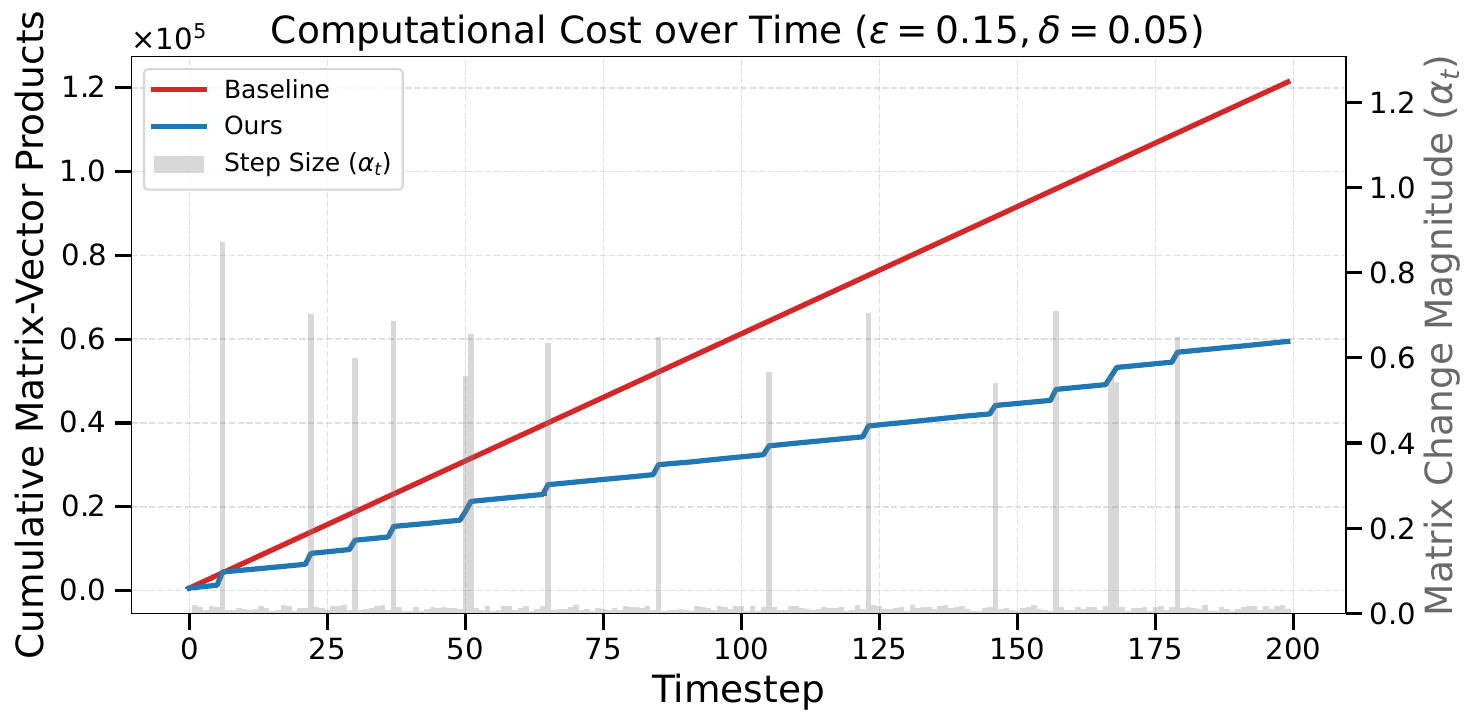}
    \caption{ Our adaptive method requires fewer MVPs to estimate traces of a sequence of matrices.}
    \label{fig:cost_vs_time}
\end{figure}

\textbf{Cost vs. Empirical Error Pareto Frontier.}
In our second experiment, we evaluate the trade-off between total computational cost (MVPs) and empirical error. To generate a single point on the trade-off graph, we first run our adaptive method for a specific target error tolerance and record its total budget. We then divide that budget evenly across all time steps for the non-adaptive baseline. Repeating this procedure across error tolerances yields the Pareto plots in Figure~\ref{fig:pareto_frontier}, which report maximum absolute error, mean absolute error, and a weighted mean error that emphasizes large matrix changes. Our method performs better on all three metrics.

\begin{figure}
    \centering
    \includegraphics[width=\linewidth]{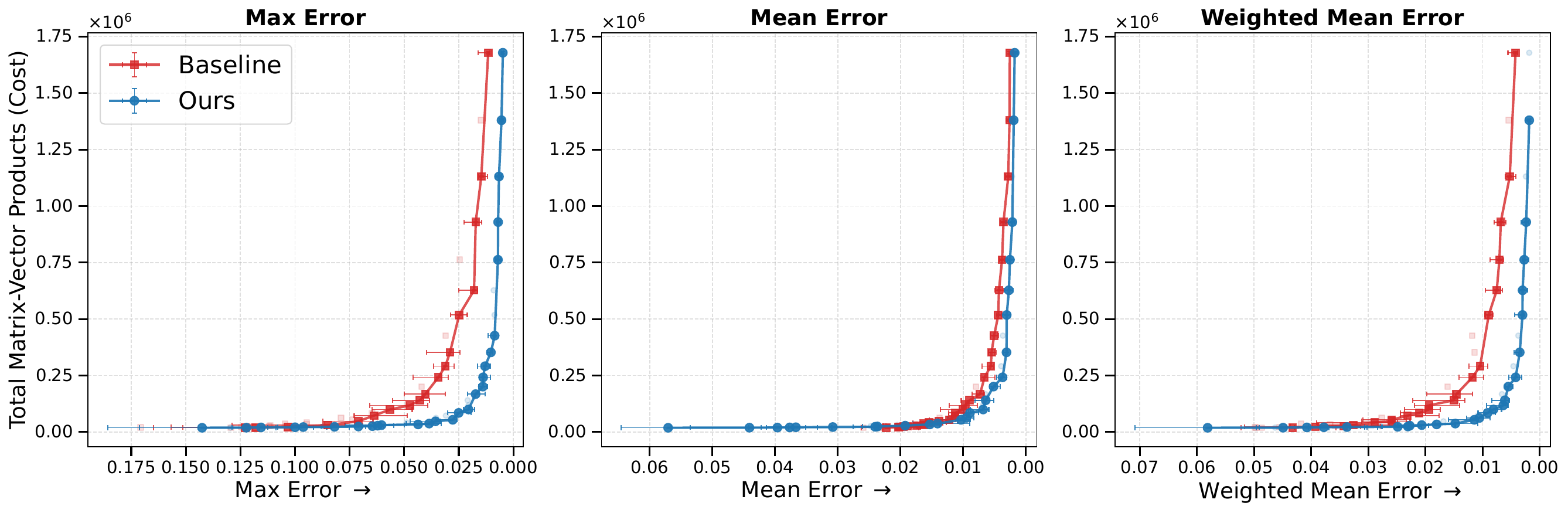}
    \caption{\textbf{Cost vs. Empirical Error Pareto Frontier.} At equivalent computational budgets, the adaptive algorithm achieves lower errors across all three metrics in these experiments.
    \label{fig:pareto_frontier}}
\end{figure}

\textbf{Hessian Trace Tracking During Neural Network Training.}
We evaluate dynamic Hessian trace estimation on a neural network training trajectory. We train a $2{,}410$-parameter multilayer perceptron for 200 steps using an SGDR (stochastic gradient descent with warm restarts) learning-rate schedule~\citep{loshchilov2017sgdr} with period $T_0=50$. We compute Hessian-vector products by backpropagation using Pearlmutter's trick~\citep{pearlmutter1994fast}. Figure~\ref{fig:hessian_trace} compares cumulative queries and trace error.

\begin{figure}
    \centering
    \includegraphics[width=\linewidth]{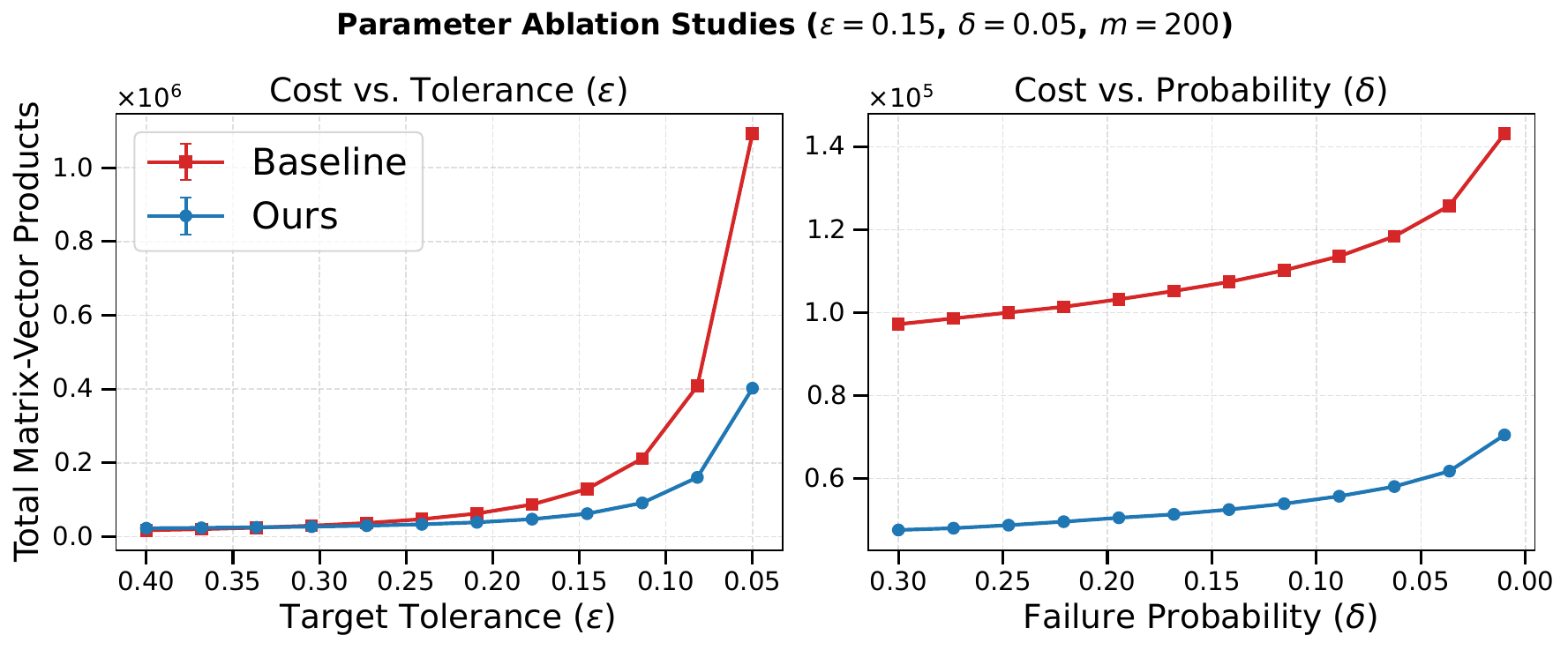}
    \caption{\textbf{Ablations.} Our method uses fewer queries than the baseline across the tested target tolerances $\epsilon$ and failure probabilities $\delta$. }
    \label{fig:ablations}
\end{figure}

\end{document}